\documentclass{article}

     \usepackage[preprint]{neurips_2019}

\usepackage[utf8]{inputenc} 
\usepackage[T1]{fontenc}    
\usepackage{hyperref}       
\usepackage{url}            
\usepackage{booktabs}       
\usepackage{amsfonts}       
\usepackage{nicefrac}       
\usepackage{microtype}      

\usepackage{latexsym,csquotes}
\usepackage{amsmath,amssymb}
\usepackage{amsthm,enumitem}
\usepackage{IEEEtrantools,todonotes}
\usepackage{amsfonts}
\usepackage{algorithm}
\usepackage[noend]{algpseudocode}

\newtheorem{theorem}{Theorem}

\DeclareMathOperator*{\argmin}{\arg\!\min}

\title{On approximating dropout noise injection}
\author{%
  Natalie Schluter
    \\
  Department of Computer Science\\
 IT University of Copenhagen\\
 Copenhagen, Denmark \\
  \texttt{natschluter@itu.dk} \\
}

\begin{document}

\maketitle

\begin{abstract}
This paper examines the assumptions of the derived equivalence between dropout noise injection and L2 regularisation for logistic regression with negative log loss. We show that the approximation method is based on a divergent Taylor expansion, making, subsequent work using this approximation to compare the dropout trained logistic regression model with standard regularisers unfortunately ill-founded to date.   Moreover, the approximation approach is shown to be invalid using any robust constraints.  We show how this finding extends to general neural network topologies that use a cross-entropy prediction layer.
\end{abstract}

\section{Introduction}
We examine the assumptions of the derived equivalence between dropout noise injection (expressed as an approximated dropout penalty term) and $L_2$ regularisation for logistic regression with negative log loss \citep{wager_etal2013}.  An important amount of empirical work has been devoted to trying to empirically extend this finding and simultaneously reconcile it with some seemingly contradictory behaviours of the noise regulariser term in question \citep{srivastava_etal2014,helmbold2015,helmbold2016}. For instance, \citet{helmbold2016} remark that dropout penalty training leads to negative parameters, even when the output is a positive multiple of the inputs, signaling co-adaptation of parameters. They also explain that for neural networks, the dropout penalty grows exponentially in the depth of the network in cases whereas the $L_2$ regulariser grows linearly.  These are both behaviours that are in strict contrast with the behaviour of the $L_2$  regulariser.

A typical tool in Machine Learning analysis is to apply the second order approximation technique.  \citet{bishop95} successfully applies it to dropout noise injection for Generalised Linear Models (GLMs) with a mean squared error loss term.  Applying this tool, a similar analysis has been carried out in the research community for logistic regression models with negative log loss, trained with dropout noise injection.  From this, it has been concluded that dropout training is ``first-order'' equivalent to training with $L_2$ regularisation, after a special scaling of the features \citep{wager_etal2013}.  In this paper, we show that the approximation method, in the particular case of logistic regression, is based on a divergent Taylor expansion.  Hence, subsequent work using this approximation to compare the dropout trained logistic regression model with standard regularisers, remains unfortunately ill-founded to date.

\paragraph{Our main result.}  Let $\mathbf{\tilde{x}}$ be the feature vector subjected to dropout,\footnote{We explain in the next section why we can express dropout this way.} $\mathbf{x}$ its non-perturbed counterpart, and $\pmb{\beta}$ the parameter vector of a logistic regression model, under training by minimising negative log loss.  Our central result is the following:
\begin{displayquote}
The Taylor series expansion of the dropout regulariser converges only for values $|\mathbf{\tilde{x}}\pmb{\beta}-\mathbf{x}\pmb{\beta}|<2\pi$.  However, dropout ensures that the distance 
\begin{equation}
|\mathbf{\tilde{x}}\pmb{\beta}-\mathbf{x}\pmb{\beta}|
\label{eq:bound}\end{equation}
remains unbounded.  Therefore,
\begin{enumerate}
\item the Taylor series expansion of the dropout regulariser does not converge in general, and
\item because of this (in addition to just the fact that the distance in (\ref{eq:bound}) is unbounded), use of the second order cut-off of the Taylor series expansion of the dropout regulariser to study dropout is ill-founded. 
\end{enumerate}
\end{displayquote}
We also expand the setting to the softmax prediction layer neural network architectures that use cross-entropy loss.  Further, we show how there cannot be any robust approach that induces meeting the assumptions of the approximation technique, thereby saving the original result. 


\paragraph{Organisation of this paper.}  We start this paper by establishing our notation: introducing the models we address, and then defining regularisation and dropout (Section \ref{sec:prelim}).  Following this,  set the context (Section \ref{sect:context}) and re-derive the quadratic approximation of previous work on dropout.  We extend the derivation to order-$k$ approximations in Section \ref{sect:higherorderTaylor}.  There, we prove an explicit expression for the $k$-th order approximation, which will be used in our main result.  Section \ref{sect:radius} provides the proof our main result.

\section{Preliminaries}\label{sec:prelim}
Let  $y\in\{0,1\}$ be a response variable, given a feature vector $\mathbf{x}\in \mathbb{R}^d$ and $A$ is the \emph{log partition function}, 
$A(z)=(1+\exp(z))$. The \emph{logistic regression} model for the distribution of $y$ given $\mathbf{x}$ is defined as 
\[p_{\pmb{\beta}}(y=1|\mathbf{x}):=p_{\pmb{\beta}}(y|\mathbf{x})=\exp (\mathbf{x\cdot\pmb{\beta}}-A(\mathbf{x\cdot\pmb{\beta}}))=\frac{1}{1+e^{-\mathbf{x}\pmb{\beta}}},\] 
where $\pmb{\beta}\in \mathbb{R}^d$ is the natural parameter (weight vector).  Correspondingly, the loss function, negative log likelihood, is given by 
\begin{IEEEeqnarray}{rl}
l_{\mathbf{x},y}(\pmb{\beta}):=& l^1_{\mathbf{x},y}(\pmb{\beta})+l^0_{\mathbf{x},y}(\pmb{\beta})\\
=&-\log (y\cdot p_{\pmb{\beta}}(y|\mathbf{x}))-(1-y)\log(1-p_{\pmb{\beta}}(y|\mathbf{x})),
\end{IEEEeqnarray}
where we have a \emph{positive example loss term} $l^1_{\mathbf{x},y}(\pmb{\beta})=-\log (y\cdot p_{\pmb{\beta}}(y|\mathbf{x}))$ and a \emph{negative example loss term} $l^0_{\mathbf{x},y}(\pmb{\beta})=-(1-y)\log(1-p_{\pmb{\beta}}(y|\mathbf{x}))$.
So finding optimal $\pmb{\beta}^*$ with respect to a dataset with $n$ examples, amounts to minimising expected loss over the dataset $\{(\mathbf{x_i},y_i)\}_{i=1}^n$: 
\begin{equation}\pmb{\beta}^*:=\argmin_{\pmb{\beta} \in \mathbb{R}^d}\sum_{i=0}^n\left(\mathbb{E}[l^1_{\mathbf{x_i},y_i}(\pmb{\beta})]+\mathbb{E}[l^0_{\mathbf{x_i},y_i}(\pmb{\beta})] \right).\label{eq:bstar}\end{equation}
In practise, one would again normalize this accumulated loss by the size of the dataset $n$ during learning.  

\paragraph{A description in terms of noised feature vectors.} For the logistic regression model, it is clear that we can consider dropout to be over features $\mathbf{x_i}$ rather than parameters $\pmb{\beta}$.  Extending to the context of neural networks requires us to generalise this description, which we will do next.  So we can speak of a \emph{dropout noised} example $\mathbf{x_i}\in \mathbb{R}^d$, which is defined as $\mathbf{\tilde{\mathbf{x}}_i}:=\mathbf{x_i}\odot\pmb{\xi_i}$ where $\odot$ is the Hadamard product (i.e., element-wise vector product) and we have each element $\xi_{ij}\in \{0,(1-\delta)^{-1}\}$, such that $\xi_{ij}$ is 0 with probability $\delta$ and otherwise $\xi_{ij}=(1-\delta)^{-1}$.  That is, components of $\mathbf{\tilde{\mathbf{x}}_i}$ are dropped to 0 following a Bernoulli$(\delta)$ distribution and are otherwise a scaled version of the respective components from $\mathbf{x_i}$.  Note that our computations will be over expected values for $\xi_{ij}$ and we do not assume the vectors $\pmb{\xi_i}$ to be constant.  The noised version of loss can then be written as (positive example term only):
\begin{equation} l^1_{\mathbf{\tilde{x}},y}(\pmb{\beta})
=-\log (p_{\pmb{\beta}}(y|\mathbf{\tilde{x}})).
\end{equation}
where we obviate explicitly writing $y$, since $y=1$ here.

For the remainder of this paper, we concentrate on the $l^1_{\mathbf{x_i},y_i}(\pmb{\beta})$ term of the loss, therefore taking an economy in notation and denoting it without the superscript.  It is sufficient to show the result of the paper for this term as do \citet{wager_etal2013}.  Indeed, minimising \[-(1-y)\log(1-p)\]
is the same as minimising \[(1-y)\log(p).\]

\paragraph{A note on extending this description to single layer feed-forward networks.}
For a simple single layer feed-forward network, we can again consider dropout to be over features $\mathbf{x_i}$.  
We introduce a dropout matrix of parameters, $\mathbf{Z_i}=\{\pmb{\xi_{i,j}}\}_{d\times q}$, 
in which each $\pmb{\xi_{i,j}}$ has dimension $d$ and $\xi_{i,j,k}\in \{0,(1-\delta)^{-1}\}$.  Then for each feature vector $\mathbf{x_i}$, we associate $q$ noising vectors $\mathbf{\tilde{\mathbf{x}}_i^j}:=\mathbf{x_i}\odot\pmb{\xi_{i,j}}$, $j\in[q]$.

For multi-class categorical cross-entropy, the response variable vector is $\mathbf{y_i}=(y_{i1},\ldots,y_{iq}), y_{ij}\in\{0,1\}, j\in [q]$, with each variable element associated to its corresponding noise vector $\pmb{\xi}_{i1},\ldots,\pmb{\xi}_{iq}$ and loss term $l_{\mathbf{\tilde{x}_i^1},y_{i1}}(\pmb{\beta}),\ldots,l_{\mathbf{\tilde{x}_i^{q}},y_{iq}}(\pmb{\beta})$ (only the positive example loss terms here), but normalised by the sum across all outputs (i.e., softmax).  Then
\begin{equation}
\pmb{\beta}^*:=\argmin_{\pmb{\beta} \in \mathbb{R}^d}\sum_{j=1}^q\sum_{i=0}^n\left(\mathbb{E}[l_{\mathbf{\tilde{x}_i^j},y_{ij}}(\pmb{\beta})] \right).\label{eq:bstar}
\end{equation}

It is straightforward to further take this description as the prediction layer of more complex network topologies.

\section{Noise regularisation for GLMs}\label{sect:context}
Recent years have seen a steady stream of work on both empirically and formally understanding the success of noise injection in avoiding overfitting.  On the formal side, \citet{webb1993} first showed that Gaussian noise injection with a least squares objective can be explicitly expressed by the original non-noisy objective with a separate regulariser.  \citet{bishop95} continued further showing that this is a Tikhonov regulariser approximately with sufficiently small noise levels.  He does so by using a second order approximation of the Taylor series expansion of the noise regulariser.  \citet{bishop95} also used this approximation to obtain an explicit expression for the correction of parameters after training, as an alternative to training with sufficiently small enough noise levels.

\citet{wager_etal2013} use the approximation technique from \citet{bishop95} to study dropout.  They express GLM log loss as log loss with a sort of drop-out regulariser.  That is, on expectation given the dropout noise, the loss can be expressed as
\begin{IEEEeqnarray}{rl}
\mathbb{E}_{\pmb{\xi}} [l_{\tilde{\mathbf{x_i}},y} (\pmb{\beta})]=&
-(y\mathbf{x_i}\cdot\pmb{\beta}+A(\mathbf{x_i}\cdot\pmb{\beta})-A(\mathbf{x_i}\cdot\pmb{\beta})-\mathbb{E}_{\pmb{\xi}}[A(\tilde{\mathbf{x_i}}\cdot\pmb{\beta})])
\nonumber\\=&\;l_{\mathbf{x_i},y}(\pmb{\beta})+R(\pmb{\beta})
\end{IEEEeqnarray}

\noindent where $R(\pmb{\beta})$ can be seen as a regulariser to be minimised with loss in lieu of training with dropout.   \cite{wager_etal2013} derived an approximation for $R(\pmb{\beta})$ that revealed an equivalence with the $L_2$ expression.  In order to do this, they resort to a polynomial approximation of $R(\pmb{\beta})$ of the expected log partition function $A$ by means of taking the first three terms (up to the second order) of the Taylor expansion of $\mathbb{E}_{\pmb{\xi}}[A(\tilde{\mathbf{x_i}}\cdot\pmb{\beta})]$ on low noise levels, following \cite{bishop95} for Gaussian noise. We will show here that the Taylor expansion in the case of dropout does not converge, making the approximation upon which this is based non-applicable.

With $\tilde{\mathbf{x_i}}\cdot\pmb{\beta}$ close to $\mathbf{x_i}\cdot\pmb{\beta}$, i.e., on low noise levels, and writing $R_2$ for the quadratic approximation, they obtain 
$R_2(\pmb{\beta})=\frac{1}{2} A''(\mathbf{x_i}\cdot\pmb{\beta})Var_{\pmb{\xi}}[\tilde{\mathbf{x_i}}\cdot\pmb{\beta}]$.

We are studying dropout, at probability $\delta$ and otherwise scaling by $\frac{1}{1-\delta}$.  To model this situation, we define the constants (with respect to dropout probability and $i$) if no parameters are dropped

\begin{IEEEeqnarray}{rl}
B_i & :=(\tilde{\mathbf{x_i}}\cdot\pmb{\beta}-\mathbf{x_i}\cdot\pmb{\beta})=\left(\frac{\mathbf{x_i}\cdot\pmb{\beta}}{1-\delta}-\mathbf{x_i}\cdot\pmb{\beta}\right)\nonumber\\
&=(\mathbf{x_i}\cdot\pmb{\beta})\left(\frac{\delta}{1-\delta}\right).
\end{IEEEeqnarray}

To take into consideration a weight dropping to 0, we introduce the boolean random variable $Y$ which follows a Bernoulli distribution with parameter $\mathbb{E}_{\pmb{\xi}}[Y]=(1-\delta)$.  We let $Z_i:=B_i\cdot Y$.  Since $B_i^n=(\mathbf{x_i}\cdot\pmb{\beta})^n\left(\frac{\delta}{1-\delta}\right)^n$
in general, for $\mathbb{E}_{\pmb{\xi}}[Z^n]$, we have
$\mathbb{E}_{\pmb{\xi}}[Z^n]=B_i^n \mathbb{E}_{\pmb{\xi}}[Y^n]=B_i^n \mathbb{E}_{\pmb{\xi}}[Y] \nonumber\\
=(\mathbf{x_i}\cdot\pmb{\beta})^n\left(\frac{\delta^n}{(1-\delta)^{n-1}}\right)$.  So, in particular for $n=2$, we have \begin{equation}
\mathbb{E}_{\pmb{\xi}}[(\tilde{\mathbf{x_i}}\cdot\pmb{\beta}-\mathbf{x_i}\cdot\pmb{\beta})^2]\nonumber
=(\mathbf{x_i}\cdot\pmb{\beta})^2\left(\frac{\delta^2}{1-\delta}\right)\label{eq:corr_wager}
\end{equation}
Where in the case of logistic regression, $A''(\mathbf{x_i}\cdot\pmb{\beta})=p_i(1-p_i)$ and $p_i=(1+e^{-\mathbf{x_i}\cdot\pmb{\beta}})^{-1}$.

\paragraph{Linear regression.}  Conveniently, for linear regression, we have
\[A(\mathbf{x_i}\cdot\pmb{\beta})=\frac{1}{2}(\mathbf{x_i}\pmb{\beta})^2.\]
Since the third order derivative disappears in this case, there is perfect equality between the $L_2$ regulariser and the dropout regulariser, which seems to give evidence to the similarity between dropout and $L_2$ regularisation in general.

\paragraph{Generalising.} Before addressing the appropriateness of this approximation in determining the relationship with other forms of standard regularisation, we generalise the approximation to order $k$.

\section{k-th order polynomial approximations}\label{sect:higherorderTaylor}
An order $k$ polynomial approximation of $\mathbb{E}_{\pmb{\xi}} [A(\tilde{\mathbf{x_i}}\cdot\pmb{\beta} )]$ is given by taking the first $k$ terms of the Taylor expansion, with $\tilde{\mathbf{x_i}}\beta$ close to $\mathbf{x_i}\beta$:

\begin{IEEEeqnarray}{rl}
\mathbb{E}_{\pmb{\xi}} [A(\mathbf{x_i}\cdot\pmb{\beta} )]&\approx \sum_{m=0}^k \frac{E_{\pmb{\xi}}[A^{(m)}(\tilde{\mathbf{x_i}}\cdot\pmb{\beta})\cdot (\tilde{\mathbf{x_i}}\cdot\pmb{\beta}-x_i\cdot\pmb{\beta} )^m] }{m!}\nonumber\\
 =&\sum_{m=0}^k \frac{A^{(m)}(\mathbf{x_i}\cdot\pmb{\beta})\cdot\mathbb{E}_{\pmb{\xi}} [(\tilde{\mathbf{x_i}}\cdot\pmb{\beta}-x_i\cdot\pmb{\beta} )^m] }{m!}
\nonumber\\
 =&\sum_{m=0}^k \frac{A^{(m)}(\mathbf{x_i}\cdot\pmb{\beta})\cdot (\mathbf{x_i}\cdot\pmb{\beta} )^m\delta^m }{m!(1-\delta)^{m-1}}\nonumber
\end{IEEEeqnarray}

Let us denote $A_k:=A^{(k)}(\mathbf{x}\cdot\pmb{\beta})$. Let $p:=(1+e^{-\mathbf{x}\cdot\pmb{\beta}})^{-1}$.  So, $A_0=p$ and $A_1=p(1-p)=:p'$, where $p'$ is the derivative of $p$.  We now provide an explicit form for $A_k$, in the case of logistic regression.  

The \emph{Triangle of numbers}, $T(n,k)$\footnote{Sequence number A019538 in the OEIS \url{oeis.org}.}, is defined as
\[T(n,k)=k!\dot S_2(n,k)\]
for $n\geq 1$ and $1\leq k\leq n$, where $S_2(n,k)$ are \emph{Stirling numbers of the second kind}.

We now show the following theorem.
\begin{theorem}
For $k\geq 1$,
\[A_k= p'\sum_{j=1}^k (-1)^{(j-1)}\cdot p^{(j-1)}\cdot T(k,j)\]
\end{theorem}
\begin{proof}
Let $A_{k,j}$ be the $j$th of the $k$ terms of $A_k$:
\[A_{k,j}:=p'(-1)^{(j-1)}\cdot p^{(j-1)}\cdot T(k,j).\]
Our proof of the theorem is by induction on $k>1$.
We have as the base case, $A_2$, because 
\[A_1=A_{1,1}=p'=p-p^2,\] 
so \[A_2=A_{2,1}+A_{2,2}=p'-2pp'.\]
Let us assume the result for $A_k$ and show that it then holds for $A_{k+1}$.  To do this, we show how the terms of $A_{k+1}$ are generated when taking the derivative of $A_k$.

The $j$th term of $A_k$ is:
\begin{IEEEeqnarray}{rl}
A_{k,j}=&p'(-1)^{j-1}p^{j-1}T(k,j)\nonumber\\
=&p^j(-1)^{j-1}T(k,j)+p^{j+1}(-1)^j T(k,j)\nonumber
\end{IEEEeqnarray}
Taking the derivative, we have
\begin{IEEEeqnarray}{rl}
(A_{k,j})'
=&jp^{(j-1)}p'(-1)^{j-1}T(k,j)+(j+1)p^jp'(-1)^j T(k,j) \nonumber\\
=&:\texttt{term}_1((A_{k,j})')+\texttt{term}_2((A_{k,j})')\nonumber\\\label{eq:A_term}
\end{IEEEeqnarray}
which is composed of two terms: the first, \texttt{term}$_1$, with $p^{j-1}p'$ and the second, \texttt{term}$_2$, with $p^jp'$.  We now show how the individual terms of $A_{k+1}$ are composed of the terms from Equation \ref{eq:A_term} for $j\in[k]$.  In particular, the first term contributes to $A_{k+1,j-1}$ and the second term contributes to $A_{k+1,j}$.

First, if $j=1$, we have 
\[A_{k+1,1}=p'\]
and thus only a contribution from the first term of $(A_{k,1})'$, which is also $p'$.  

If $j=k+1$, we have
\[A_{k+1,k+1}=p'(-1)^kp^kT(k+1,k+1)\]
and thus only a contribution from the second term of $(p'A_{k,k})'$, since
\begin{IEEEeqnarray}{rl}
(p'A_{k,k})'&=kp^{(k-1)}p'(-1)^{k-1}T(k,k)
+(k+1)p^kp'(-1)^k T(k,k) \nonumber\\
=&kp^{(k-1)}p'(-1)^{k-1}T(k,k)
+p^kp'(-1)^k T(k+1,k+1) \nonumber \\
=&kp^{(k-1)}p'(-1)^{k-1}T(k,k)+A_{k+1,k+1}.\nonumber
\end{IEEEeqnarray}

By definition of the Triangle of numbers, and the recurrence relation $S_2(k+1,j)=j!S_2(k,j)+S_2(k,j-1)$ of the Stirling numbers of the second kind, we have
$T(k+1,j)=j!S_2(k+1,j)=j!\left(jS_2(k,j)+S_2(k,j-1)\right)$.

If $1<j<k+1$,
\begin{IEEEeqnarray}{rl}
A_{k+1,j}
=& p'(-1)^{(j-1)}\cdot p^{(j-1)}\cdot j!\left(jS_2(k,j)+S_2(k,j-1)\right) \nonumber\\
=& p'(-1)^{(j-1)}\cdot p^{(j-1)}\cdot  j! \cdot j \cdot S_2(k,j)
+p'(-1)^{(j-1)}\cdot p^{(j-1)}\cdot j! \cdot S_2(k,j-1)\nonumber\\
=& p'(-1)^{(j-1)}\cdot p^{(j-1)}\cdot  j\cdot T(k,j)
+p'(-1)^{(j-1)}\cdot p^{(j-1)}\cdot j\cdot T(k,j-1)\nonumber\\
=& \texttt{term}_1((A_{k,j})')+\texttt{term}_2((A_{k,j-1})')\nonumber
\end{IEEEeqnarray}

So, we have shown that
\begin{IEEEeqnarray}{rl}
A_{k+1}=&\sum_{j=1}^{k+1} A_{k+1,j}\nonumber\\
=&\texttt{term}_1((p'A_{k,1})')+ \texttt{term}_2((p'A_{k,k})') +\left( \sum_{j=2}^{k} \texttt{term}_1((p'A_{k,j})')+\texttt{term}_2((p'A_{k,j-1})')\right)  \nonumber\\
=& \sum_{j=1}^k \left(  \texttt{term}_1((p'A_{k,j})')+\texttt{term}_2((p'A_{k,j})')\right)\nonumber\\
=& \sum_{j=1}^k (p'A_{k,j})'=(p'A_k)'
\end{IEEEeqnarray}
as desired.
\end{proof}

Let $R_k(\pmb{\beta})$ denote the $k$-th ($k\geq 2$) order polynomial approximation to $R$ using $A_k$, close to the point $\mathbf{x_i}\cdot\pmb{\beta}$.  So, that 

\begin{eqnarray}
R_k(\pmb{\beta}):=\sum_{j=2}^k A_j\frac{E[(\tilde{\mathbf{x_i}}\cdot\pmb{\beta}-\mathbf{x_i}\cdot\pmb{\beta})^j]}{j!}
\label{eq:kth_order_approx}
\end{eqnarray}

Since the set of the Stirling numbers $S(k',j')$ for $k'\leq k$ and $j'\leq j$ can be computed in $O(k\cdot j)$ time with dynamic programming, replacing the expression of $A_k$ from the theorem in Equation (\ref{eq:kth_order_approx}) yields an $O(k\cdot j)$ algorithm for calculating $R_k$.

\section{The appropriateness of the Taylor expansion}\label{sect:radius}
We note that each component $x_{ij}$ of $\tilde{\mathbf{x_i}}\beta$ is in the set $\{0, \frac{x_{ij}\beta_j}{1-\delta}\}$.  In other words, the perturbed value $\tilde{\mathbf{x_i}}\beta$ needs to be within the radius of convergence $r$ of the Taylor expansion, $|\tilde{\mathbf{x_i}}\beta-\mathbf{x_i}\beta|<r$.  However, the value $\mathbf{x_i}\beta$ is completely unconstrained: the only assumption made thus far (following \citet{wager_etal2013}) being that the levels of noise are ``low''.  
For additive Gaussian noise, noise levels can be adjusted so that this assumption is satisfied, so long as there is some $r$ for which this the Taylor series expansion converges.  However, dropout causes arbitrarily many values of $\tilde{\mathbf{x_i}}$ to be zero, thus $\tilde{\mathbf{x_i}}\beta$ can be very close to zero, even though $\mathbf{x_i}\beta$ can be arbitrarily large.

We now show that the radius of convergence is bounded, which excludes any meaningful study of dropout using Taylor series expansion of the expected log partition function.  Recall $R(\pmb{\beta})=\mathbb{E}_{\pmb{\xi}}[A(\tilde{\mathbf{x_i}}\cdot\pmb{\beta})]-A(\mathbf{x_i}\cdot\pmb{\beta})$ and that $p_i=(1+\exp(\mathbf{x_i}\pmb{\beta}))^{-1}$.  Also recall the explicit expression for $A_k$ from Theorem 1.

\begin{theorem} \label{th:conv}
\begin{equation}R_\infty(\pmb{\beta}):=\lim_{k\rightarrow \infty}\sum_{j=2}^k A_j\frac{E[(\tilde{\mathbf{x_i}}\cdot\pmb{\beta}-\mathbf{x_i}\cdot\pmb{\beta})^j]}{j!}\label{eq:taylor}\end{equation}
converges to $R(\pmb{\beta})$ if and only if
\[|\mathbf{\tilde{\mathbf{x_i}}}\pmb{\beta}-\mathbf{x_i}\pmb{\beta}|<2\pi.\]
\end{theorem}
That is, Theorem \ref{th:conv} states that the radius of convergence of $R_\infty(\pmb{\beta})$ is $2\pi$ for examples in the dataset that the un-noised model will classify as true.  Since $\mathbf{x_i}\beta$ is unbounded, and certainly larger than $2\pi$ arbitrarily much of the time during training, we have a proof of Theorem \ref{th:conv}.

\begin{proof}
We want to test for convergence of the ratio of factors, and if convergent, calculate the convergence ratio.  For this we use the root test.  Using the $\limsup$ takes into account $p\rightarrow 1$.
\begin{IEEEeqnarray}{l}
\limsup_{k\rightarrow \infty} \sqrt[k]{\left|\frac{A_k}{k!}\right|}=\limsup_{k\rightarrow \infty} \sqrt[2k]{\left|\frac{A_{2k}}{(2k)!}\right|}
\nonumber\\=\limsup_{k\rightarrow\infty} \sqrt[2k]{\left|\frac{1}{(2k)!}\sum_{j=1}^{2k}(-p)^j\cdot T(2k,j)\right|}\nonumber\\
=\limsup_{k\rightarrow\infty} \sqrt[2k]{\left|\frac{1}{(2k)!}\sum_{j=1}^{2k}(-p)^j\cdot j\cdot S(2k,j)\right|}\nonumber\\
=\limsup_{k\rightarrow\infty} \sqrt[2k]{\left|\frac{1}{(2k)!}\sum_{j=1}^{2k}(-1)^j\cdot j\cdot S(2k,j)\right|}\nonumber\\
=\limsup_{k\rightarrow\infty} \sqrt[2k]{\left|\frac{1}{(2k)!}B_{2k}\right|}
\nonumber\\=\limsup_{k\rightarrow\infty} \sqrt[2k]{\left|\frac{1}{(2k)!}\frac{2\cdot (2k)!}{(2\pi)^{2k}}\zeta(2k)\right|}
\\=\frac{1}{2\pi}
\label{eq:ratio_conv}
\end{IEEEeqnarray}
where is $B_k$ is the $k$th Bell number. We used the well-established fact that
\[B_{2k}\thicksim \frac{2\cdot (2k)!}{(2\pi)^{2k}}\zeta (2k) \]
asymptotically as $k$ grows, where $\zeta(2k)$ is the Riemann zeta function, which tends to 1 as $k\rightarrow \infty$.
This concludes the proof.
\end{proof}

However, if we can bound $x_{ij}\beta_j$, for example to be less than $\frac{2\pi}{d}$, where $d$ is the dimension of the features vector $\mathbf{x_{i}}$, the approximation in Equation \ref{eq:taylor} holds, as does the original result that dropout noise approximates $L_2$ regularisation for logistic regression.  However, in order to do this, we must 
\begin{enumerate}
    \item keep the features $x_{ij}$ small, which could be carried out by a number of different normalisation or input vector scaling approaches, but we must also
    \item keep the weights $\beta_j$ low, which is the effect of $L_2$ regularisation.  However, $L_2$ regularisation is what dropout was supposed to approximate in the first place.
\end{enumerate}
There are a number of engineering tricks which have an effect of reducing the potential final magnitude of the weight vector. For instance, gradient clipping \citep{pascanu_etal2013,mikolov2012} ensures that the norm of the gradient does not surpass a certain predetermined threshold.  One could be tempted to apply a technique like reparameterisation weight normalisation, which expresses the weight vector in terms of two factors: a unit vector $v$ and a magnitude $g$ \citep{Salimans_Kingma2016}.  Then attempt to hold the magnitude term $g$ constant under the desired restricted magnitude.  However, this is unlikely to work, as we explain now.

In fact, there are no techniques that can both ensure convergence of the learning method, adequate performance of the model, and that restricts the magnitude of the weight vector, within this dropout framework.  It is straightforward to see why this is:  the probability $p(y=1|x)=\frac{1}{(1+e^{-x\beta})}$ should tend towards 1 on positive examples.  If it doesn't, and there is potential for improvement, the model will continue try to improve (i.e., will not yet converge).  However, this means that $-x\beta$ should approach the value 1, which only happens if $x$ is reciprocally as large in magnitude as $\beta$.  So the closer to 0 we make $x$, the larger the magnitude of $\beta$.  Also, $p(y=1|x)$ should tend toward 0 on negative examples, which can only happen if $-x\beta$ tends toward infinity, which requires $\beta$ to grow arbitrarily large in magnitude.

\section{Concluding remarks}
We have shown that recent studies on dropout expressed as a regulariser on the original loss function are ill-founded.  This re-opens a possible line of research to formally establish whether there is a concrete connection between regularisation and the practice of dropout training.  Future work could start in the reverse direction:  for which variations of dropout could such an approximation technique, that we have shown to be inapplicable in general, be valid?

\bibliography{noise}
\bibliographystyle{acl_natbib}
\end{document}